\documentclass[a4paper]{article}
\usepackage[a4paper,margin=1.2in]{geometry}
\usepackage[labelfont=bf]{caption}

\usepackage{natbib, amsfonts,booktabs,color,epsfig,hyperref,url}

\usepackage{amsmath}
\usepackage{amssymb} % for mathbb
\usepackage{amsthm}
\usepackage{bbm} % for math bold
\usepackage{bm} % for math bold
\usepackage{graphicx}
\usepackage{mathtools} % allows DeclarePairedDelimiter
\usepackage{subcaption}
\usepackage{tikz} %tikzpicture
\usetikzlibrary{calc} % needed for coord calculation

\newtheorem{theorem}{Theorem}[section]
\newtheorem{lemma}[theorem]{Lemma}
\newtheorem{proposition}[theorem]{Proposition}
\newtheorem{remark}[theorem]{Remark}

\newtheorem{assumption}[theorem]{Assumption}

\newcommand{\R}{\mathbb{R}}

\newcommand{\di}{\,\text{d}}
\newcommand{\Ex}{\mathsf{E}}
\newcommand{\Norm}{\mathcal{N}}
\DeclareMathOperator*{\argmax}{argmax}
\DeclareMathOperator*{\argmin}{argmin}

\newcommand{\N}{\mathbb{N}}
\newcommand{\Var}{\mathrm{Var}}

\DeclarePairedDelimiter\abs{\lvert}{\rvert}%
\DeclarePairedDelimiter\inner{\langle}{\rangle}%
\DeclarePairedDelimiter\norm{\lVert}{\rVert}%

\begin{document}

\title{On the Estimation of Entropy in the FastICA Algorithm}
\author{Elena Issoglio, Paul Smith\footnote{Corresponding author: mmpws@leeds.ac.uk}, and Jochen Voss}
\maketitle

\begin{abstract}
  The fastICA method is a popular dimension reduction technique used
  to reveal patterns in data. Here we show both theoretically and in
  practice that the approximations used in fastICA can result in
  patterns not being successfully recognised. We demonstrate this
  problem using a two-dimensional example where a clear structure is
  immediately visible to the naked eye, but where the projection
  chosen by fastICA fails to reveal this structure. This implies that
  care is needed when applying fastICA. We discuss how the problem
  arises and how it is intrinsically connected to the approximations
  that form the basis of the computational efficiency of
  fastICA.
\end{abstract}



\section{Introduction}\label{secIntroduction}

Independent Component Analysis (ICA) is a well-established and popular
dimension reduction technique that finds an orthogonal projection of data
onto a lower-dimensional space, while preserving some of the original
structure. ICA is also used as a method for blind source separation and is
closely connected to projection pursuit.
We refer the reader to
\citet{hyvarinen2004independent}, \citet{hyvarinen1999fast} and
\citet{stone2004independent}
for a comprehensive overview of the mathematical
principles underlying ICA and its applications in a wide variety of practical
examples.

In ICA, the projections that are determined to be ``interesting'' are
those that maximise the non-Gaussianity of the data, which can be
measured in several ways. One quantity for this measurement that is
used frequently in the ICA literature is entropy. For distributions
with a given variance, the Gaussian distribution is the one which
maximises entropy, and all other distributions have strictly smaller
entropy. Therefore, our aim is to find projections which minimise the
entropy of the projected data. Different methods are available for
both the estimation of entropy and the optimisation procedure, and
have different speed-accuracy trade-offs.

A widely used method to perform ICA in higher dimensions is fastICA
\citep{hyvarinen2000independent}. This method has found applications
in areas as wide ranging as facial recognition
\citep{draper2003recognizing}, epileptic seizure detection
\citep{yang201581} and fault detection in wind turbines
\citep{farhat2017fast}. Recent works on extensions of the algorithm
can be seen in \citet{miettinen2014deflation},
\citet{ghaffarian2014automatic} and \citet{he2017large}. The fastICA
method uses a series of substitutions and approximations of the
projected density and its entropy. It then applies an iterative
scheme for optimising the resulting contrast function (which is an
approximation to negentropy). Because of its popularity in many
areas, analysis and evaluation of the strengths and weaknesses of the
fastICA algorithm is crucially important. In particular, we need to
understand both how well the contrast function estimates entropy and
the performance of the optimisation procedure.

The main strength of the fastICA method is its speed, which is
considerably higher than many other methods. Furthermore, if the data
is a mixture of a small number of underlying factors, fastICA is often
able to correctly identify these factors. However, fastICA also has
some drawbacks, which have been pointed out in the
literature. \citet{learned2003ica} use test problems from
\citet{bach2002kernel} with performance measured by the Amari error
\citep{amari1996new} to compare fastICA to other ICA methods. They
find that these perform better than fastICA on many examples.
Focussing on a different aspect, \citet{wei2014spurious} investigates
issues with the convergence of the iterative scheme employed by
fastICA to optimise the contrast function. In \citet{wei2017study} it
is shown that the two most common fastICA contrast functions fail to
de-mix certain bimodal distributions with Gaussian mixtures, although
some other contrast function choices (related to classical kurtosis
estimation) may give reliable results within the fastICA framework.

In this article we identify and discuss a more fundamental problem
with fastICA. We demonstrate that the approximations used in fastICA
can lead to a contrast function where the optimal points no longer
correspond to directions of low entropy.

\begin{figure}
  \centering
  \includegraphics{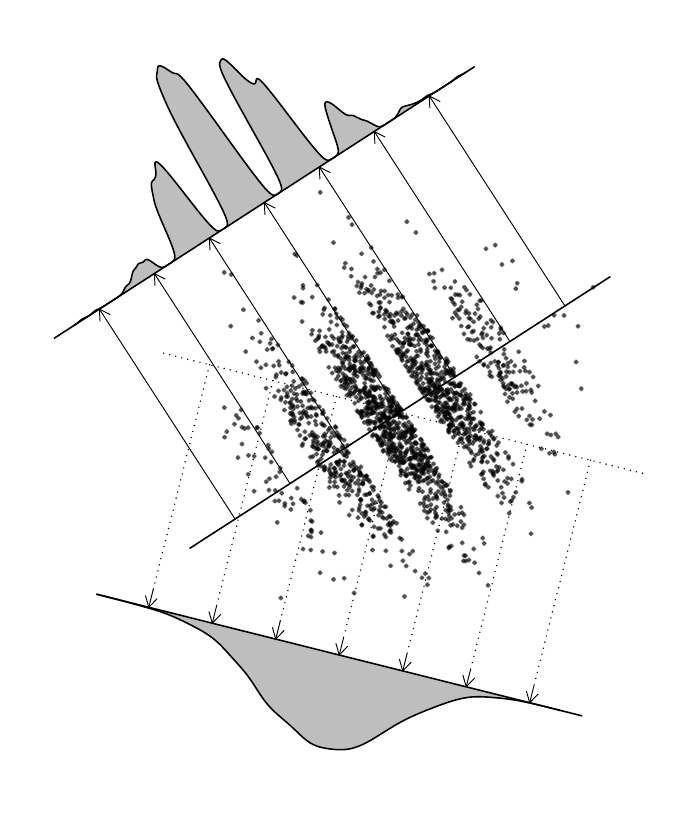}
  \caption{Scatter plot of original data with densities of the
    projected data in the direction obtained by $m$-spacing ICA (solid
    line) and fastICA (dotted line). Kernel density estimation was
    used to obtain the marginal densities shown.\label{figPoints}}
\end{figure}

To illustrate the effect discussed in this paper, we consider the
example shown in Figure~\ref{figPoints}. In this example,
two-dimensional samples are generated from a two-dimensional normal
distribution, conditioned on avoiding a collection of parallel bands
(see Section~\ref{secExample} for details). This procedure produces a
pattern which is obvious to the bare eye, and indeed the projection
which minimises entropy (solid lines) exposes this pattern. In
contrast, the fastICA contrast function seems unable to resolve the
pattern of bands and prefers a direction with higher entropy which
does not expose the obvious pattern (dotted lines).

We remark that
this failure by fastICA to recover the obvious structure is relatively
robust in this example. Changing the parameters used in the fastICA
method does not significantly change the outcome, and the underlying
structure is still lost. It is also worth mentioning here that the
example dataset was very simple to obtain and no optimisation was
performed to make the fastICA method perform poorly.

To obtain the projection indicated by the solid lines in Figure~\ref{figPoints} we
used the $m$-spacing method \citep{beirlant1997nonparametric} for
entropy estimation, combined with a standard optimisation technique.
The $m$-spacing entropy approximation is shown to be consistent in
\citet{hall1984limit} and converges to true entropy for large sample
size. While the $m$-spacing entropy approximation theoretically makes
for an excellent contrast function for use in ICA, it is relatively
slow to evaluate.

Our main theoretical contribution helps explain why fastICA in the
example in Figure~\ref{figPoints} performs poorly. To obtain the
contrast function in the fastICA method a surrogate to the true
density is first obtained, and then it is approximated through several
steps to increase computational speed. In Section~\ref{secProofs} we
show convergence results for the approximation steps, and conclude
that the accuracy loss occurs at the initial stage where the real
density is replaced by the surrogate one.  This is highlighted in
Figure~\ref{figOptimalMSpacingDenn}, which shows the estimated density
(solid line) along the direction that exposes the pattern
(\textit{i.e.}\ the solid line in Figure~\ref{figPoints}). The dotted
line shows the surrogate density in this same direction that the
fastICA method uses in its approximation to entropy.
Figure~\ref{figOptimalFastIcaDenn} shows the analogue for the
direction found by fastICA (\textit{i.e.}\ the dotted line in
Figure~\ref{figPoints}).  This figure motivates this paper by
highlighting the area where the approximations used in fastICA diverge
from the true values. The two densities in Figure~\ref{figOptimalMSpacingDenn}
are very different to one-another, and this error
propagates through the fastICA method to the approximation used for
entropy.  Note that the solid lines in Figure~\ref{figOptimalDenn}
show the same estimated densities as given in Figure~\ref{figPoints}.

\begin{figure}
  \centering
  \begin{subfigure}[b]{0.45\textwidth}
    \includegraphics{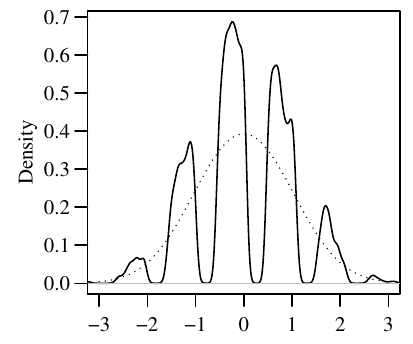}
    \caption{Densities along the optimal direction found using the $m$-spacing method}
    \label{figOptimalMSpacingDenn}
  \end{subfigure}
  \hfill
  \begin{subfigure}[b]{0.45\textwidth}
    \includegraphics{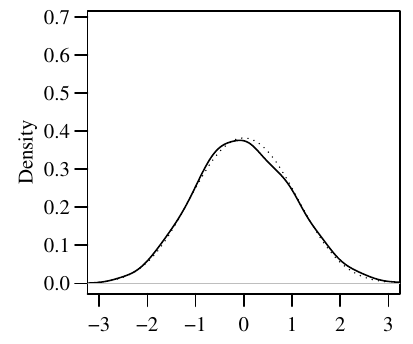}
    \caption{Densities along the optimal direction found using the \texttt{fastICA} method}
    \label{figOptimalFastIcaDenn}
  \end{subfigure}
  \caption{Plots showing an estimate of the density of the projected data (solid
    line), and the surrogate density $f_0$ used in the fastICA method
    (dotted line), for two different projections of the
    data. Panel~\ref{figOptimalMSpacingDenn} corresponds to the
    direction of highest entropy, found using $m$-spacing, and
    Panel~\ref{figOptimalFastIcaDenn} corresponds to the direction
    found by fastICA. These two directions are shown by the solid line
    and dotted line respectively in Figure~\ref{figPoints}.\label{figOptimalDenn}}
\end{figure}

This paper is structured as follows. In Section~\ref{secEntropy}, entropy and
negentropy are introduced alongside associated estimates.
In Section~\ref{secFastIca} we describe the
fastICA method. Section~\ref{secProofs} contains some
proofs which help to understand where errors are introduced in
the fastICA method. Section~\ref{secExample} contains details
on the example given in Figure~\ref{figPoints}. Some concluding
remarks can be found in Section~\ref{secConclusion}.
The code to produce the figures in this paper can be found at
\url{https://github.com/pws3141/fastICA_code}.

\section{Entropy and Negentropy}
\label{secEntropy}

The aim of the fastICA method is to efficiently find a
projection of given data which minimises entropy.
Suppose we have a one-dimensional random variable $X$ with density
$f\colon \R \rightarrow [0, \infty)$. Then the entropy $H$ of the
distribution of $X$ is defined to be
\begin{equation}\label{eqDifferentialEntropy}
  H[f] := - \int_{\R} f(x) \log f(x) \di x,
\end{equation}
whenever this integral exists. We use square brackets to indicate that
$H$ is a functional, taking the function $f$ as its argument. In the
special case of a Gaussian random variable with variance~$\sigma^2$,
the entropy can be calculated explicitly and it takes the value
$\eta(\sigma^2)$ given by
\begin{equation}\label{eqGaussEntropy}
  \eta(\sigma^2)
  := \frac{1}{2} \bigl(1 + \log(2\pi \sigma^2)\bigr).
\end{equation}
It is known that this is an upper bound for entropy, namely the
entropy of any random variable with variance $\sigma^2$ will belong to
the interval $(-\infty, \eta(\sigma^2)]$ \citep[see, for
example,][]{cover2012elements}. The negentropy $J$ is defined as
\begin{align*}
        J[f] &:= \eta(\sigma^2) - H[f],
\end{align*}
where $\eta( \sigma^2)$ is given by~\eqref{eqGaussEntropy}. This
implies that $J[f]\in[0,\infty)$. Negentropy is zero when the density
is Gaussian, and strictly greater than zero otherwise.

As the definition of entropy involves the integral of the density, the
estimation of entropy or negentropy from data is non-trivial.
For a survey of different methods to estimate entropy from data,
see~\citet{beirlant1997nonparametric}.   As an example,
we consider here the $m$-spacing estimator, originally given
in~\citet{vasicek1976test}.  Suppose we have a sample of
one-dimensional points, $y_1, y_2, \ldots, y_n \in \R$, from a
distribution with density $f$, and $y_{(1)}, y_{(2)}, \ldots, y_{(n)}$
is the ordering such that
$y_{(1)} \leq y_{(2)} \leq \cdots \leq y_{(n)}$. Define the
$m$-spacing difference to be $\Delta_m y_i = y_{(i+m)} - y_{(i)}$ for
$m \in \{3, \ldots, n - 1\}$ and $i \in \{1, 2, \ldots, n - m\}$. The
$m$-spacing approximation for entropy $H[f]$ for the sample
$y = (y_1, y_2, \ldots, y_n)$ is given by
\begin{equation}\label{m-spacing}
  H_{m, n}(y)
  = \frac{1}{n} \sum_{i=1}^{n-m} \log\Bigl(\frac{n}{m}\, \Delta_m y_i \Bigr)
    - \digamma(m) + \log(m),
\end{equation}
where $\digamma(x) = -\frac{\di}{\di x} \Gamma(x)$ is the digamma
function. This is a realisation of the general $m$-spacing formula
given in \citet{hall1984limit}.
This approximation tends to the
true value of entropy under certain conditions and so for a ``large enough'' number
of points should be comparable to the true value.
This method has been used
previously within an ICA method by
\citet{learned2003ica}.  While the methods provides consistent estimates for the
entropy, it is computationally expensive.  The main contribution to computational
cost comes from the need to sort the sample $y$ in increasing order.

The fastICA method provides a more efficient way to estimate
negentropy $J[f]$ by using a series of approximations and substitutions both for
$f$ and for $J[\cdot]$ to obtain a surrogate for negentropy $J[f]$ which is then
subsequently maximised. The reason behind these substitutions is to reduce
computational cost, but the drawback is that the resulting approximation may
be very different from the true contrast function.

\section{The fastICA Algorithm}
\label{secFastIca}

In this section we describe the fastICA method of
\citet{hyvarinen2000independent}.  The theory behind this method was
originally introduced in \citet{hyvarinen1998new}, although here we
adjust the notation to match the rest of this paper. We will mention
explicitly where our notation differs from \citet{hyvarinen1998new}
and \citet{hyvarinen2000independent}.  We will write `fastICA' when we
are discussing the theoretical method, and `\texttt{fastICA}' when we
are discussing the \texttt{R} implementation from the \texttt{fastICA}
CRAN package \citep{RfastICA}.

The fastICA method to obtain the first loading from data
$\tilde D \in \R^{n \times \tilde p}$ follows the steps given below.
Following the usual convention, the rows of $\tilde D$ denote
observations, the columns denote variables.
\begin{enumerate}
        \item[i.] Whiten the data to obtain $D \in \R^{n \times p}$ with
        $p = \min(\tilde p, n - 1)$, such that $\frac{1}{n - 1} D^\top D = I_p$
        \citep[Section 5.2]{hyvarinen2000independent};
        \item[ii.] Iteratively find the optimal projection $w^\ast$, given by
                \begin{equation}\label{eqIterative}
                    w^\ast = \argmax_{w \in \R^p,\, w^\top w = 1}
                        \hat J^\ast(Dw),
                \end{equation}
                where $\hat J^\ast$ is an approximation to negentropy,
                given in equation~\eqref{eqHatJStarY} below.
\end{enumerate}
If more than one loading is required, Step~ii.\ is repeated for each
subsequent new direction, with the added constraint that $w$ must be
orthogonal to the previously found directions.  This can be
implemented within the fastICA framework using Gram-Schmidt
orthogonalisation \citep[Section 6.2]{hyvarinen2000independent}.  This
is known as the \textit{deflation} fastICA method. There is also a
\textit{parallel} fastICA method that finds all loadings concurrently,
although in this paper we only consider the deflation approach.

In the literature regarding fastICA it is often the convergence of the
iterative method to solve \eqref{eqIterative} that is examined. It can
be shown, for example in \citet{wei2014spurious}, that in certain
situations this iterative step fails to find a good approximation for
$w^\ast$. In contrast, here we consider the mathematical substitutions
and approximations used in the derivation of $\hat J^\ast(Dw)$.
Assumption~\ref{assumptionGK} introduces the technical assumptions
given in \citet[Sections 4 and 6]{hyvarinen1998new}, using slightly
adjusted notation.

\begin{assumption}\label{assumptionGK}
        Let $G_i$, $i = 1, 2, \ldots, I$ be functions that do not grow faster than
        quadratically.
        Let $\varphi(\cdot)$ denote the density of a standard Gaussian
        random variable and assume that there are
        $\alpha_i, \beta_i, \gamma_i, \delta_i$, $i = 1, 2, \ldots, I$, such that
        the functions
        \begin{align}\label{eqKx}
                K_i(x) := \frac{G_i(x) + \alpha_i x^2 + \beta_i x + \gamma_i}{\delta_i}
        \end{align}
        satisfy
        \begin{subequations}\label{eqKxConditions}
        \begin{align}
                \int_{\R} K_i(x) K_j(x) \varphi(x) \di x
                        &= \mathbbm{1}_{\{i = j\}};\text{ and, }\label{eqKxConditions-1}\\
                \int_{\R} K_i(x) x^k \varphi(x) \di x &= 0,
                        \text{ for } k = 0, 1, 2,\label{eqKxConditions-2}
        \end{align}
        \end{subequations}
        for $i,\,j = 1, 2, \ldots, I$, where $\mathbbm{1}_{\{i = j\}} = 1$
        if $i = j$ and zero otherwise.
\end{assumption}

The functions $G_i$ are given as $\bar G_i$ in \citet{hyvarinen1998new} and
as $G_i$ in \citet{hyvarinen2000independent}.  The functions $K_i$ are
described in \citet[Section 6]{hyvarinen1998new} and are called $G_i$
there.

The \texttt{fastICA} algorithm only implements the case $I = 1$.  In
this case, the function $G_1$ can be chosen nearly arbitrarily so long as it
does not grow faster than quadratically: It is easy to show that for every $G$
which is not exactly equal to a second order polynomial, a function $K_1$
can be found that satisfies the conditions given
in~\eqref{eqKxConditions} by choosing suitable $\alpha_1$, $\beta_1$,
$\gamma_1$ and $\delta_1$.  For general $I \in \mathbb{N}$, specific $G_i$,
$i = 1, 2, \ldots, I$ must be chosen for the
conditions~\eqref{eqKxConditions} to hold.  With $I = 2$, the
functions $G_1(x) = x^3$ and $G_2(x) = x^4$ are proposed in the
literature \citep[Section $7$]{hyvarinen1998new} and seem to be useful
in practice, even though these functions violate the growth condition
from Assumption~\ref{assumptionGK}.  We have not found any examples of
specific functions $G_i$ that satisfy~\eqref{eqKxConditions} for
$I > 2$ in the fastICA literature.%

Let $w\in\R^p$ with $\|w\|=1$ and let
$y = (y_1, y_2, \ldots, y_n) = Dw \in \R^n$ be the data projected
onto~$w$.  Since the data has been whitened, $y$ has sample mean~$0$
and sample variance~$1$.  Further, let $f\colon \R\to \R$ be the
unknown density of the population-level-whitened and projected data.  Then $f$
satisfies $\int f(x) \,dx = 1$, $\int x \, f(x) \, dx = 0$ and
$\int x^2 \, f(x) \, dx = 1$.  We need to estimate the negentropy
$J[f]$ using the data $y_1, \ldots, y_n$.  Define
\begin{equation}\label{eqConstraints}
  c_i
  := \Ex_f K_i(X)
  = \int f(x) K_i(x) \di x
\end{equation}
for all $i\in\{1, \ldots I\}$.
For $I = 1$, setting $K(x) := K_1(x)$, $G(x) := G_1(x)$ and $c := c_1$,
the derivation of the contrast function used in the
fastICA method then consists of the following steps:
\begin{enumerate}
        \item Replace $f$ by a density $f_0$ given by
\begin{equation}\label{eqf0}
        f_0(x) = A \exp\Bigl(\kappa x + \zeta x^2 + a K(x) \Bigr),
\end{equation}
for all $x \in \R$.  The constants $A$, $\kappa$, $\zeta$ and
$a$ are  chosen to minimise negentropy (and hence maximise entropy)
under the constraints $\int f_0(x) K(x) \di x = c$.
In Proposition~\ref{propf0MaxEntropy} we will show that $J[f_0] \leq J[f]$.
        \item Approximate $f_0$ by $\hat f_0$ defined as
                \begin{equation}\label{eqhatf0}
                        \hat f_0(x) = \varphi(x) \bigl(1 + c K(x)\bigr)
                \end{equation}
                for all $x \in \R$.
                In Theorem~\ref{theoremSupremumf0Hatf0} we will show $J[\hat f_0] \approx J[f_0]$.
        \item Approximate $J[\hat f_0]$ by second order Taylor expansion,
                \begin{equation}\label{eqHatJHatf0}
                        \hat J[\hat f_0] = \frac{1}{C} \bigl(
                                        \Ex_f G(Y) - \Ex_\varphi G(Z)
                                        \bigr)^2,
                \end{equation}%
                where $Y$ is a random variable with density $f$, $Z \sim \Norm(0,1)$,
                and $C$ some constant.
                Note that, maybe surprisingly, $Y$ has density $f$, not $f_0$.
                In Proposition~\ref{propNegentropyProportional} we will show that
                $\hat J[\hat f_0] \approx J[\hat f_0]$.
        \item  Use Monte-Carlo approximation for the expectations in~\eqref{eqHatJHatf0},
                \textit{i.e.}\ use
                \begin{equation}\label{eqHatJStarY}
                        \hat J^\ast(y) = \Bigl(\frac{1}{n} \sum_{j = 1}^n G(y_j)
                        - \frac{1}{L} \sum_{j = 1}^L G(z_j)\Bigr)^2,
                \end{equation}
                where $z_1, \ldots, z_L$ are samples
                from a standard Gaussian and $L$ is large.
                Here $\hat J^\ast(y) \approx C \hat J[\hat f_0] $.
\end{enumerate}
The restriction to $I = 1$ here removes a summation from Step~1.\ and
Step~2.\, therefore simplifying Step~3.\ and the associated
estimations in Step~4.  Theoretically these steps can be completed for
arbitrary $I \in \mathbb{N}$, although in this case a closed-form
version equivalent to Step~3.\ is much more complicated.

The approximation~\eqref{eqHatJStarY} to the negentropy used in
fastICA dramatically decreases the computational time needed to find
ICA projections.  Unlike the $m$-spacing estimator introduced in
Section~\ref{secEntropy}, the approximation $\hat J^\ast(Dw)$ is a
simple Monte-Carlo estimator and does not require sorting of the data.
The algorithm to solve~\eqref{eqIterative} also benefits from the fact
that an approximate derivative of $w \mapsto \hat J^\ast(Dw)$ can be
derived analytically.

The steps in this chain of approximations are illustrated in
Figure~\ref{figFastIcaApproximations} and we will investigate the
approximation more formally in Section~\ref{secProofs}.  In Step~1.\
of the procedure, we do not obtain a proper approximation, but have an
inequality instead: $f$ is replaced with a density $f_0$ such that
$J[f_0] \leq J[f]$.  As a result, the $w$ which maximises $J[f_0]$ can
be very different from the one which maximises~$J[f]$.  In contrast,
Steps~2.\ and~3.\ are proper approximations and in
Section~\ref{secProofs} we prove convergence of $\hat f_0$ to $f_0$
for Step~2.\ and of $\hat J[\hat f_0]$ to $J[f_0]$ for Step~3.\ in the
limit $\|c\| \to 0$, where $c = (c_1, \ldots, c_I)$.  Step~4.\ is a
simple Monte-Carlo approximation exhibiting well-understood behaviour.
From the above discussion, it seems sensible to surmise that the loss
of accuracy in fastICA is due to the surrogate used in Step~1.\ above.

\begin{figure}
  \centering
  \begin{tikzpicture}
  % first row
  \node [left] at (0,-1) {$f$};
  \draw [->] (0,-1) --(1.5,-1);
  \node [right] at (1.5,-1) {$J[f]$};
  % connector
  \node [rotate=-90] at (2.05,-1.75) {$\geq$};

  % second row
  \node [left] at (0,-2.5) {$f_0$};
  \draw [->] (0,-2.5) --(1.5,-2.5);
  \node [right] at (1.5,-2.5) {$J[f_0]$};
  % connector
  \node [rotate=-90] at (2.05,-3.25) {$\approx$};
  \node [rotate=-90] at (-0.25,-3.25) {$\approx$};

  % third row
  \node [left] at (0,-4) {$\hat f_0$};
  \draw [->] (0,-4) --(1.5,-4);
  \node [right] at (1.5,-4) {$J[\hat f_0]$};
  % connector
  \node [rotate=-90] at (2.05,-4.75) {$\approx$};

  %fourth row
  \node [right] at (1.5,-5.5) {$\hat J[\hat f_0] = \frac{1}{C} \bigl(\Ex_f G(Y) -
        \Ex_\varphi G(Z)\bigr)^2$};
  % connector
  \node [rotate=-90] at (2.05,-6.25) {$\approx$};

  % fifth row
  \node [right] at (1.5,-7) {$\frac{1}{C} \, \hat J^*(y)$, where 
                $\hat J^*(y)  = 
                    \bigl(\frac{1}{n} \sum_{j = 1}^n G(y_j) 
                        - \frac{1}{L} \sum_{j = 1}^L G(z_j)\bigr)^2$};

  % second row
  %\draw [<-] (0.5,-0.9) --(0.5,0.2);
  %\node [align=center, above] at (0.5,0.2) {Surrogate density \hspace{2cm}};
  %\draw [<-] (2.125,-0.9) --(2.125,0);
  %\node [align=center, above] at (2.125,0) {\hspace{5cm} Approximation of surrogate density};

  % negentr
  % f
  %\draw [->] (-0.3,-1.4) --(-0.3,-2.75);
  %\node [below] at (-0.3,-2.75) {$J[f]$};
  % f_0
  %\draw [->] (2.85,-1.4) --(2.85,-2.75);
  %\node [below] at (2.85,-2.75) {$J[\hat f_0]$};
  %\draw [->] (2.85,-3.5) --(2.85,-5);
  %\node [below] at (2.85,-5) {$\hat J[\hat f_0]  \propto \bigl(\Ex G(Y) - \Ex
                        %G(Z)\bigr)^2$};
  %\draw [->] (4.2,-3.2) --(4.8,-3.2);
  %\node [right] at (5.5,-3.2) {$\hat J_{\hat f_0}(X) \propto \bigl(\Ex G(X) -
                        %\Ex G(\nu)\bigr)^2$};

  %% third row
  %\draw [<-] (3,-2.05) --(4,-2.05);
  %\node [align=center, right] at (4,-2.05) {Negentropy of \\ density approximation};
  %\draw [<-] (3,-4.25) --(4,-4.25);
  %\node [align=center, right] at (4,-4.25) {Approximation of \\ $J[\hat f_0]$};
  %\draw [<-] (3,-6.375) --(4,-6.375);
  %\node [align=center, right] at (4,-6.375) {Approximation of $\hat J[\hat f_0]$ \\ from data};

  % result
  %\draw [->] (2.85,-5.75) --(2.85,-7);
  %\node [align=center, below] at (2.85,-7) {$\hat J^*(y)  = \bigl(\frac{1}{N}
                        %\sum_{j = 1}^N G(y_j) - \frac{1}{l} \sum_{j = 1}^l G(z_j)\bigr)^2$};

  % iterative step
 % \draw [->] (2.85, -7.75) --(2.85,-8.75);
 % \node [align=center, below] at (2.85,-8.75) {Iterative step to \\approximate $y^*$};

  % Ex G
  %\node [align = center] at (6.5,1) {$\Ex G(Y) \approx \frac{1}{N} \sum_{j =
                        %1}^N G(y_j)$\\
  %$\Ex G(Z) \approx \frac{1}{l} \sum_{j = 1}^l G(z_j)$};
  %\draw [->] (7, 0.1) --(8,-0.5);
\end{tikzpicture}
  \caption{Approximations used in fastICA: The fastICA contrast
    function $\hat J^\ast(y)$ is used in place of negentropy $J[f]$.
    Note that the first step involves an inequality rather than
    an approximation.\label{figFastIcaApproximations}}
\end{figure}

We conclude this section with a few simple observations: Using
\eqref{eqKxConditions}, \eqref{eqKx} and the fact that $X$ and $Z$ are
standardized we find
\begin{align*}
  c_i
  &= \Ex_f K_i(X) \\
  &= \Ex_f K_i(X) - \Ex_\varphi K_i(Z) \\
  &= \Ex_f \left( \frac{G_i(X) + \alpha_i X^2 + \beta_i X + \gamma_i}{\delta_i} \right)
    - \Ex_\varphi \left( \frac{G_i(Z) + \alpha_i Z^2 + \beta_i Z + \gamma_i}{\delta_i} \right) \\
  &= \frac{\Ex_f G_i(X) + \alpha_i 1 + \beta_i 0 + \gamma_i}{\delta_i}
    - \frac{\Ex_\varphi G_i(Z) + \alpha_i 1 + \beta_i 0 + \gamma_i}{\delta_i} \\
  &= \frac{\Ex_f G_i(X) - \Ex_\varphi G_i(Z)}{\delta_i}.
\end{align*}
Thus, the fastICA objective function (ignoring the final Monte Carlo
approximation) satisfies $\hat J[\hat f_0] \propto c^2$ for the case
$I=1$, considered above, and
$\hat J[\hat f_0] \propto \sum_{i=1}^I c_i^2$ in the general case.
Thus, fastICA can only see the data through the~$c_i$.  If the data
are approximately Gaussian, we have
$\Ex_f G_i(X) \approx \Ex_\varphi G_i(Z)$ and $c_i \approx 0$ for
all~$i$ and thus $\hat J[\hat f_0] \approx 0$, but the opposite
implication does not hold.  This is in contrast to the true
negentropy, which satisfies $J[f] = 0$ if and only if $f$ is Gaussian.

A first consequence of this argument is that projections where the
true distribution is Gaussian will look `uninteresting' to fastICA:
for these directions $w$ the objective function $\hat J^\ast(Dw)$ will
be small and the search for the maximum in~\eqref{eqIterative} will be
driven away from these directions.  This is particularly relevant
since for high dimensional data, where the search volume is vast,
projections along most directions are close to Gaussian
\citep{diaconis1984asymptotics, von1997sudakov}, so fastICA will be
able to exclude much of the search volume.  Conversely, if
$\hat J[\hat f_0]$ and thus $\|c\|$ is large, the projected density
$f$ is not Gaussian and by maximising (an approximation to)
$\hat J[\hat f_0]$, the fastICA method will find directions which are
`interesting'.  But the above discussion also shows that optima can be
missed when $\hat J[\hat f_0]$ is small, but the projected density~$f$
is still far from Gaussian.  This is the case we are concerned with in
this paper and thus we assume $\|c\| \approx 0$ when we consider the
fastICA approximations in detail in the next section.

\section{Approximations used in the fastICA Method}
\label{secProofs}

In this section, we investigate the validity of the approximation
given in Section~\ref{secFastIca}. We consider Step~1.\ in
Proposition~\ref{propf0MaxEntropy}, Step~2.\ in
Theorem~\ref{theoremSupremumf0Hatf0}, and Step~3.\ in
Proposition~\ref{propNegentropyProportional}.  Throughout this
section, we consider arbitrary $I \in \mathbb{N}$ for completeness.

We first introduce some assumptions, in addition to
Assumption~\ref{assumptionGK}, that are required for the mathematics
in this section to hold.

\begin{assumption}\label{as1K}
        There exists
        $\varepsilon > 0$ such that for all $h \in \R^I$ with $h^\top h <
        \varepsilon$, we have
        \begin{equation}\label{eqAsmpKBoundedBelow}
                h^\top K(x) \geq -\frac{1}{2}
        \end{equation}
        for all $x \in \R$, where
        $K(x) = \bigl( K_1(x), K_2(x), \ldots, K_I(x) \bigr)$.
        In addition, there exists a function $M\colon \R \rightarrow \R$ such that
        \begin{subequations}\label{eqKxConditionM}
        \begin{align}
            \sum_{i=1}^I \sum_{j=1}^I \sum_{k=1}^I
                    \lvert K_i(x) K_j(x) K_k(x) \rvert &\leq M(x)
                    \text{ for all } x \in \R,
                    \text{ and}\label{eqKxConditionM-1}\\
                \int_{\R} \varphi(x) M(x) \di x &=: \tilde M < \infty.
                    \label{eqKxConditionM-2}
        \end{align}
        \end{subequations}
\end{assumption}

Note that under the condition that each $G_i$ does not grow faster than
quadratically (given in Assumption~\ref{assumptionGK}), we can always find some
positive constants $B_i$, $i = 1, 2, \ldots, I$ such that
\begin{equation}\label{eqKiBoundingConstraint}
        \lvert K_i(x) \rvert \leq B_i (1 + x^2),
\end{equation}
for all $x \in \R$.
Note also that for $I = 1$ the condition
given by~\eqref{eqAsmpKBoundedBelow} that
there exists an $\varepsilon > 0$ such that for all $h \in [0,\varepsilon)$,
we have $h K(x) \geq -1 / 2$ is satisfied as follows.
Let $\alpha, \beta,
\gamma, \delta$ be parameters for which~\eqref{eqKxConditions} holds.
Then,~\eqref{eqKxConditions} holds also for $\alpha, \beta, \gamma, -\delta$.
Moreover, since $G$ does not grow faster than quadratically, $\alpha x^2$ is the
dominant term in $K(x)$ as $x \rightarrow \pm \infty$. Therefore, to ensure
that~\eqref{eqAsmpKBoundedBelow} holds it is enough to choose $\delta$ or $- \delta$
such that the sign is the same as that of $\alpha$.

\subsection{Step 1.}

  We start our discussion by considering Step~1.\ of the approximations
  described in Section~\ref{secFastIca}.  We prove that the
  distribution which maximises entropy for given values of
  $c_1, \ldots, c_I$ is indeed of the form~\eqref{eqf0} and thus
  that we indeed have $J[f_0] \leq J[f]$.%

\begin{proposition}\label{propf0MaxEntropy}
Let $f$ be the density of the population-level-whitened data projected in some direction
(thus with zero mean and unit variance). Recall $c_i$ is defined by~\eqref{eqConstraints}.
The density $f_0$ that maximises entropy in the set
\[
        \Bigl\{g\colon \R \rightarrow \R \,;\, g \text{ is a density function, and }
                    \int_\R g(x) K_i(x) \di x = c_i,\, i = 1, 2, \ldots, I \Bigr\},
\]
is given by,
\begin{equation}\label{eqMaxEntropyDensity}
        f_0(x) = A \exp\left(\kappa x + \zeta x^2 + \sum_{i = 1}^I a_i K_i(x)\right)
\end{equation}
for some constants $\kappa$, $\zeta$, $A$ and $a_i,\, i = 1, 2, \ldots, I$ that
depend on $c_i, \, i = 1, 2, \ldots, I$. It follows from this that
$J[f_0] \leq J[f]$.
\end{proposition}

\begin{proof}
  We use the method of Lagrange multipliers in the calculus of
  variations \citep[see, for example,][]{lawrence1998partial}
  to find a necessary condition for the density that maximises entropy
  given the constraints on mean and variance, and
  in~\eqref{eqConstraints}. Let $F[\cdot]\colon C^2 \rightarrow \R$ be
  a functional of the function $g\colon \R \rightarrow \R$, with
  $g \in C^2$, where $C^2$ is the set of all twice continuously
  differentiable functions. Then, the functional derivative
  $\delta F/ \delta g\colon \R \rightarrow \R$ is explicitly defined by
\begin{equation}\label{eqDefnFunctionalDerivative}
                \int_{\R} \frac{\delta F}{\delta g}(x) \phi(x) \di x
                        := \frac{\di}{\di \varepsilon} F[g + \varepsilon \phi]
                    \Bigr\vert_{\varepsilon = 0}
                        = \lim_{\varepsilon \downarrow 0}
                    \Bigl(\frac{F[g + \varepsilon \phi] - F[g]}
                            {\varepsilon} \Bigr),
\end{equation}
for any function $\phi \in C^2$. The right-hand side
of~\eqref{eqDefnFunctionalDerivative} is known as the G\^{a}teaux
differential $\di F(g; \phi)$.  Define the inner product of two
functions by $\inner{g, h} := \int_{\R} g(x) h(x) \di x$, with norm
$\lVert g \rVert_{\text{L}^2} := \inner{g, g}^{\frac{1}{2}} =
\bigl(\int_{\R} g(x)^2 \di x\bigr)^{\frac{1}{2}}$.  We want to solve
the following system of equations

\begin{align*}
        \begin{cases}
                U[g](x) &:= \frac{\delta}{\delta g} H[g] +
                    \lambda_1 \frac{\delta}{\delta g} V[g] +
                    \lambda_2 \frac{\delta}{\delta g} P[g]
                    + \lambda_3 \frac{\delta}{\delta g} Q[g]
                    + \sum_{i=1}^I \nu_i \frac{\delta}{\delta g} R_i[g] = 0;\\
                V[g] &= 0;\\
                P[g] &= 0;\\
                Q[g] &= 0;\\
                R_i[g] &= 0,
        \end{cases}
\end{align*}
where $\lambda_1, \lambda_2, \lambda_3, \nu_i, i = 1,\ldots, I$ are some real
numbers, $H[g]$ is entropy as given in~\eqref{eqDifferentialEntropy}, and
\begin{align*}
        V[g] &:= \Var[g] - 1 = \int_{\R} g(x) x^2 \di x
            - \Bigl(\int_{\R} g(x) x \di x \Bigr)^2 - 1;\\
        P[g] &:= \int_{\R} g(x) \,\text{d}x - 1;\\
        Q[g] &:= \int_{\R} g(x) x \,\text{d}x;\\
        R_i[g] &:= \int_{\R} g(x) K_i(x) \,\text{d}x - c_i.
\end{align*}
Using~\eqref{eqDifferentialEntropy} and~\eqref{eqDefnFunctionalDerivative}
the term with $H$ gives,
\begin{align*}
    \langle \frac{\delta H}{\delta g}, \phi \rangle
        &= -\frac{\di}{\di \varepsilon} \int
            \bigl(g(x) + \varepsilon \phi(x)\bigr)
            \log\bigl(g(x) + \varepsilon \phi(x)\bigr)
                \di x \Bigr\vert_{\varepsilon = 0}\\
        &= - \int \Bigl( g(x) \frac{\phi(x)}{g(x) + \varepsilon \phi(x)}
            + \phi(x) \log\bigl(g(x) + \varepsilon \phi(x)\bigr)
            + \varepsilon \phi(x) \frac{\phi(x)}{g(x) + \varepsilon \phi(x)}
                \Bigr)\di x \Bigr\vert_{\varepsilon = 0}\\
        &= - \int \bigl(1 + \log g(x)\bigr) \phi(x) \di x\\
        &= \inner{-1 -\log g(x), \phi}.
\end{align*}
Now, looking at $V[g]$ and using the constraint $Q[g] = 0$ we get,
\begin{align*}
\inner{\frac{\delta V}{\delta g}, \phi}
    &= \frac{\di}{\di \varepsilon} \biggl(\int \bigl(g(x) + \varepsilon \phi(x)\bigr)
        x^2 \di x - \Bigl(\int \bigl(g(x) + \varepsilon \phi(x)\bigr) x \di x\Bigr)^2
        - 1 \biggr) \biggr\vert_{\varepsilon = 0}\\
    &= \int \phi(x) x^2 \di x - 2\Bigl(\int \phi(x) x \di x
        \cdot \int g(x) x \di x\Bigr)\\
    &= \inner{x^2, \phi} - 2 \inner{x, \phi} \cdot Q[g]\\
    &= \inner{x^2, \phi}.
\end{align*}
Let $L[\cdot]\colon C^2 \rightarrow \R$ be of the form
$L[g] = \int g(x) l(x) \di x - k$
for some function $l\colon \R \rightarrow \R$, and some constant $k \in \R$. Then it
is easy to check that $\inner{\dfrac{\delta L}{\delta g}, \phi} = \inner{l, \phi}$
and therefore $\dfrac{\delta P}{\delta g} = 1$, $\dfrac{\delta Q}{\delta g} = x$ and
$\dfrac{\delta R_i}{\delta g} = K_i$.
Putting this into the equation for $U[g]$, we have
\begin{align*}
        U[g](x) = - 1 - \log g(x) + \lambda_1 + \lambda_2 x^2 + \lambda_3 x
            + \sum_{i = 1}^I \nu_i K_i(x).
\end{align*}
Setting $U[g] = 0$ and solving for $g$ gives,
$g(x) = f_0(x) = \exp[\lambda_1 - 1 + \lambda_2 x^2 + \lambda_3 x + \sum_{i=1}^I \nu_i K_i(x)]$
which is~\eqref{eqMaxEntropyDensity} with
$A = \exp(\lambda_1 - 1)$, $\kappa = \lambda_3$, $\zeta = \lambda_2$ and $a_i
= \nu_i$, $i = 1, \ldots, I$.
Note that the constants $A,\,\kappa,\,\zeta$, and $a_i$ depend on $c_i$ indirectly
through the constraints on the $K_i$ expressed as $R_i[g] = 0$.
\end{proof}

\begin{remark}\label{remarkJfJf0Bounds}
  It is possible to specify a density $f$ such that in some limit,
  $H[f] \rightarrow \infty$ whilst $H[f_0]$ remains bounded and thus
  $\bigl| J[f] - J[f_0] \bigr| \rightarrow \infty$, with $f_0$ the
  density given in~\eqref{eqf0}. That is, in Step~1.\ of the fastICA
  method given in Section~\ref{secFastIca}, the difference between the
  true negentropy and the surrogate negentropy can be arbitrarily
  large. For example, set the density $f$ to be a mixture of two
  independent uniform densities, i.e.
  \[
    f(x) = \frac{1}{2} \bigl( g(x;\, -1 - \varepsilon, -1) +
                                            g(x;\, 1, 1 + \varepsilon) \bigr)
  \]
  where $\varepsilon \in \R$ and $g(\cdot\,;\, a, b)$ is the density
  function of a Uniform distribution in the interval $[a, b]$. Then we
  have expectation and variance given by
  \[
          \Ex_f X = 0; \quad \Var_f X = 1 + \varepsilon +
                                          \frac{\varepsilon^2}{3}.
  \]
  As the support of $g(\cdot\,;\, -1 - \varepsilon, -1)$ is disjoint
  from that of $g(\cdot\,;\, 1, 1 + \varepsilon)$, the entropy is
  given by,
  \[
          H[f] = \frac{1}{2} \bigl( H[g(\cdot\,;\, -1 - \varepsilon, -1)] +
                          H[g(\cdot\,;\, 1, 1 + \varepsilon)] \bigr)- \log(2).
  \]
  We have $H[f] \rightarrow -\infty$ as $\varepsilon \rightarrow 0$, since $f$ tends
  to a pair of Dirac deltas. Also,
  \begin{equation}\label{eqLimitOfDiracDeltasExpectation}
          \Ex_f K_i(x) =: c_i \rightarrow
                  \frac{1}{2} \bigl(K_i(-1) + K_i(1) \bigr),
  \end{equation}
  as $\varepsilon \rightarrow 0$. With $f_0$ as in~\eqref{eqf0},
  \begin{equation}\label{eqExf0Ki}
          c_i = \int K_i(x) f_0(x) \di x,
  \end{equation}
  and,
  \begin{align*}
          H[f_0] &= \int f_0(x) \log(A) \di x
                      + \int f_0(x) \bigl( \eta x + \kappa x^2 +
                                  \sum_{i = 1}^I a_i K_i(x) \bigr) \di x\\
                  &= \log(A) + \eta \Ex_{f_0} X + \kappa \Ex_{f_0} X^2
                      + \sum_{i=1}^I a_i \Ex_{f_0} K_i(x)\\
                  &= \log(A) + \kappa + \sum_{i = 1}^I a_i c_i.
  \end{align*}
  Therefore, for $H[f_0]$ to be unbounded from below as $\varepsilon \rightarrow 0$
  we would require some $\kappa \rightarrow -\infty$,
  $a_i \rightarrow -\infty$ or $A \rightarrow 0$,
  as $c_i$ is bounded by~\eqref{eqLimitOfDiracDeltasExpectation} and
  Assumption~\ref{as1K}. However, this can not occur whilst $f_0$
  satisfies~\eqref{eqExf0Ki}.
\end{remark}

\subsection{Step 2.}

We now switch our attention to Step~2.\ of the approximations.  As
discussed in Section~\ref{secFastIca}, we consider the case where
$c \rightarrow 0$.  The first step of our analysis is to identify the
behaviour of the constants in the definition of $f_0$ as $c\to 0$.  We
then prove some auxiliary results before concluding our discussion of
Step~2.\ in Theorem~\ref{theoremSupremumf0Hatf0}.

\begin{proposition}\label{propNormSquaredIFT}
Suppose Assumption~\ref{assumptionGK} is satisfied, and let
$A, \kappa, \zeta, a_1, \dots, a_I$ be defined as in
Proposition~\ref{propf0MaxEntropy}, as functions of $c$. Then
\begin{align*}
        \begin{split}
        \gdef\bigOc{\mathcal{O}(\norm{c}^2)}
                A - \frac{1}{\sqrt{2 \pi}} &= \bigOc\\
                \kappa &= \bigOc\\
                \zeta + \frac{1}{2} &=  \bigOc\\
                a_i - c_i &= \bigOc, \quad i = 1, 2, \ldots, I,
        \end{split}
\end{align*}
as $\|c\| \rightarrow 0$.
\end{proposition}

\begin{proof}
Define $x = (c_1, \ldots, c_I)^\top\in \R^I$ and
$y = (A, \kappa, \zeta, a_1, \ldots, a_I)^\top\in \R^{I+3}$.
Furthermore, let $F\colon \R^I \times \R^{I+3} \rightarrow \R^{I+3}$ be given by
\begin{equation*}
F(x, y) =
    \begin{pmatrix}
        \int f_0(x) \di x - 1\\
        \int f_0(x) x \di x\\
        \int f_0(x) x^2 \di x - 1\\
        \int f_0(x) K_1(x) \di x - c_1\\
        \vdots\\
        \int f_0(x) K_I(x) \di x - c_I
    \end{pmatrix},
\end{equation*}
where $f_0$ is given in~\eqref{eqMaxEntropyDensity} and $K_i$
in~\eqref{eqKx}.  Then, for the points $x_1 = (0, \ldots, 0)^\top$ and
$y_1 = (\tfrac{1}{\sqrt{2 \pi}}, 0, -\tfrac{1}{2}, 0, \ldots,
0)^\top$, we have $F(x_1, y_1) = 0$.

Assuming $F$ is twice differentiable, we use the Implicit Function
Theorem \citep[see, for example,][]{oliveira2014implicit} around
$(x_1, y_1)$.  First, we need to show $D_{y} F(x_1, y_1)$ is
invertible. We have
\begin{equation*}
D_{y} F(x_1, y_1) =
\begin{pmatrix}
    M & 0\\
    0 & -I_I
\end{pmatrix},
\quad \text{with, }
M =
\begin{pmatrix}
    \sqrt{2} & 0 & 1\\
    0 & 1 & 0\\
    1 & 0 & 4
\end{pmatrix}.
\end{equation*}
Therefore, $D_{y} F(x_1, y_1)$ is non-singular, and so the Implicit
Function Theorem holds.
There exist some open set $\mathcal{U} \subset \R^I$ and a
unique continuously differentiable function
$g\colon \mathcal{U} \rightarrow \R^{I + 3}$ such that $g(x_1) = y_1$ and
$F\bigl(x, \,g(x) \bigr) = 0$ for all $x \in\mathcal{U}$.
Then,
\begin{equation}\label{eqIFTDg}
        D g(x)
            = -D_y F\bigl(x, \,g(x)\bigr)^{-1}
                D_x F\bigl(x, \, g(x) \bigr).
\end{equation}
As $g$ is continuous in the set $\mathcal{U}$, there exists some $\varepsilon > 0$,
such that for all $c \in \mathcal{U}$ with $\norm{c} < \varepsilon$,
$g(x_1 + c) = y_1 + d$ for some $d \in \R^{I+3}$.
Using Taylor series we can expand $g$ around $x_1 = 0 \in \R^I$ to obtain
$g(x_1 + c) = g(x_1) + D g(x_1) \, c
        + \mathcal{O}(\norm{c}^2)$, and
\[
        D g(x_1) = \frac{d + \mathcal{O}(\norm{c}^2)}{c}.
\]
Putting this together with~\eqref{eqIFTDg} at $x = x_1$ and rearranging gives,
\[
        d
            = - D_{y} F(x_1, y_1)^{-1}
                D_{x} F(x_1, y_1) \,c
                + \mathcal{O}(\norm{c}^2).
\]
Now, since
\[
        D_x F(x_1, y_1) =
            \begin{pmatrix}
                0 & \cdots & 0\\
                0 & \cdots & 0\\
                0 & \cdots & 0\\
                & I_I &
            \end{pmatrix}
            \in \R^{(I+3) \times I},
\]
one easily obtains that
\[
        d =
            \begin{pmatrix}
                0 & \cdots & 0\\
                0 & \cdots & 0\\
                0 & \cdots & 0\\
                & I_I &
            \end{pmatrix}
            c
            + \mathcal{O}(\norm{c}^2),
\]
and so,
\[
        y_1 + d
            = \begin{pmatrix}
                \frac{1}{\sqrt{2 \pi}}\\
                0\\
                -\frac{1}{2}\\
                c_1\\
                \vdots\\
                c_I
            \end{pmatrix}
                + \mathcal{O}(\norm{c}^2), \text{ as } c \rightarrow 0.
\]
This completes the proof.
\end{proof}

We now define the following functions $y(\cdot)$ and $r(\cdot)$ for future use.
Let $y\colon \R \rightarrow \R$ be given by
\begin{equation}\label{eqDefnFunctionY}
        y(x) := \kappa x + (\zeta + \frac{1}{2}) x^2 + \sum_{i=1}^I a_i K_i(x),
\end{equation}
and $r\colon \R \rightarrow \R$ given by
\begin{equation}\label{eqDefnFunctionR}
        r(x) := e^x - 1 - x.
\end{equation}
Using these definitions, we can write $f_0$, given in
Proposition~\ref{propf0MaxEntropy}, as
\begin{equation}\label{eqf0Rearranged}
        f_0(x) = \varphi(x) \cdot \sqrt{2 \pi} A e^{y(x)}.
\end{equation}
The following lemmas are two technical results needed in the proof of Theorem~
\ref{theoremSupremumf0Hatf0}.

\begin{lemma}\label{lemmaSupremumConvexFunctions}
Let $g\colon \R \rightarrow \R$ and $l\colon \R \rightarrow \R$ be any
functions and $h\colon \R \rightarrow \R_{+}$ be convex with $h(0) = 0$.
Then,
\begin{equation*}
    \sup_{x \in \R} \bigl \lvert l(x) h(\varepsilon g(x)) \bigr \rvert
        \leq \varepsilon \,\sup_{x \in \R} \bigl \lvert l(x) h(g(x)) \bigr \rvert
\end{equation*}
for all $\varepsilon \in [0, 1]$.
\end{lemma}

\begin{proof}
As $h$ is convex, for all $\lambda \in [0,1]$ and for all $x, y \in \R$, we have
$h\bigl(\lambda x + (1 - \lambda) y\bigr) \leq \lambda h(x) + (1 - \lambda) h(y)$.
Let $\varepsilon \in [0,1]$. Then, substituting $\lambda = \varepsilon$,
$x = g(x)$ and $y = 0$, we have
$h\bigl(\varepsilon \, g(x)\bigr) \leq \varepsilon \,h\bigl(g(x)\bigr)$,
for all $g(x) \in \R$, as $h(0) = 0$. Noticing that $h$ maps to the positive real
line allows to conclude.
\end{proof}

\begin{lemma}\label{lemmaRxSupremumBound}
        Let $r\colon \R \rightarrow \R_{+}$ be given as in~\eqref{eqDefnFunctionR}.
Then,
\begin{equation}\label{eqRxBound}
    r(\varepsilon \,y) \leq \varepsilon^2 r(y),
        \text{ for all $y \geq 0$, and for all $\varepsilon \in [0, 1]$.}
\end{equation}
Moreover, for any function $l\colon \R \rightarrow \R$, we have
\begin{equation*}
    \sup_{x \in \R} \Bigl \lvert l(x)
        r \bigl(\varepsilon (1 + x^2)\bigr) \Bigr \rvert
        \leq \varepsilon^2 \sup_{x \in \R} \Bigl \lvert l(x) r(1 + x^2) \Bigr \rvert.
\end{equation*}
\end{lemma}

\begin{proof}
We will use the Taylor expansion of the exponential around $0$
for both the
left-hand and right-hand side of~\eqref{eqRxBound}.
The left-hand side gives,
\begin{align*}
    r(\varepsilon\, y)
        &= \exp(\varepsilon\,y) - 1 - \varepsilon\,y\\
        &= \sum_{n = 0}^\infty \frac{\varepsilon^n}{n!} \,y^n - 1 -
            \varepsilon\,y, \quad
            \text{absolutely convergent for all } \varepsilon y \in \R\\
        &= \varepsilon^2\, \Bigl(\sum_{n = 2}^\infty
            \frac{\varepsilon^{n-2}}{n!} y^n\Bigl)
        \intertext{and the right-hand side of~\eqref{eqRxBound} gives,}
    \varepsilon^2\, r(y)
        &= \varepsilon^2\, \Bigl(\sum_{n = 0}^\infty \frac{1}{n!} \,y^n
            - 1 - y\Bigl)
        = \varepsilon^2\, \Bigl(\sum_{n = 2}^\infty \frac{1}{n!} \,y^n\Bigr).
\end{align*}
Putting these two results together,
\begin{equation*}
r(\varepsilon\,y) - \varepsilon^2\, r(y)
    = \varepsilon^2 \Bigl(\sum_{n = 2}^\infty
        \frac{1}{n!}\,y^n (\varepsilon^{n-2} - 1)\Bigr)
    \leq 0,
\end{equation*}
        as $\varepsilon^n - 1 \leq 0$ for all $\varepsilon \in [0,1]$ and
        $n \in \N_{+}$.
This proves~\eqref{eqRxBound}.

Let $l\colon \R \to \R$ be some function. Then, as $r$ maps to the
positive real line and using~\eqref{eqRxBound} with $y = 1 + x^2$, we have
$ \lvert l(x) r\bigl(\varepsilon (1 + x^2)\bigr) \rvert
    \leq \varepsilon^2 \lvert l(x) r(1 + x^2) \rvert$,
for all $x \in \R$. Taking the supremum over the real line we conclude.
\end{proof}

        We now consider the error term between the density $f_0$ that maximises
        entropy, and its estimate~$\hat f_0$.%

\begin{theorem}\label{theoremSupremumf0Hatf0}
        Suppose we have functions $K_i$, $i = 1, 2, \ldots, I$ that satisfy
        Assumptions~\ref{assumptionGK} and~\ref{as1K}.
        Let $f_0$ be given as in Proposition~\ref{propf0MaxEntropy}, and
        let $\hat f_0$ be given by
        \[
                \hat{f}_0(x) = \varphi(x)\Bigl(1 + \sum_{i=1}^I c_i K_i(x)\Bigr).
        \]
        Then,
        \begin{align*}
                \sup_{x \in \R} \bigl \lvert e^{\delta x^2} \bigl(f_0(x) - \hat f_0(x) \bigr)
            \bigr \rvert = \mathcal{O}(\norm{c}^2) \text{ as } c \rightarrow 0,
        \end{align*}
        for all $\delta < 1/2$.
\end{theorem}

\begin{proof}
        Let $\varphi(x) = (2 \pi)^{1/2} e^{-x^2 / 2}$ be the density of a standard
Gaussian random variable and let the function $g\colon \R \rightarrow
\R$ be defined by
\[
        g(x) := \frac{f_0(x) - \hat{f}_0(x)}{\varphi(x)}.
        \]
Then, with $y\colon \R \rightarrow \R$ as defined in~\eqref{eqDefnFunctionY} and
using~\eqref{eqf0Rearranged} we get,
\begin{align*}
        g(x)
        &= \sqrt{2 \pi} A \exp\bigl(y(x)\bigr) -
        \bigl(1 + \sum_{i=1}^I c_i K_i(x)\bigr)\nonumber\\
        &= \sqrt{2 \pi} A \Bigl(\exp\bigl(y(x)\bigr) - 1 - y(x)\Bigr)
                + \sqrt{2 \pi}A\bigl(1 + y(x)\bigr)
                - (1 + \sum_{i=1}^I c_i K_i(x))\nonumber\\
            &\quad+ \sqrt{2 \pi} A \bigl(\sum_{i=1}^I c_i K_i(x)
        - \sum_{i=1}^I c_i K_i(x)\bigr).\nonumber
\end{align*}
Rearranging this using the function $r\colon \R \rightarrow \R$ given in~
\eqref{eqDefnFunctionR} and by expanding $y(x)$ gives,
\begin{align*}
        g(x)
        &= \sqrt{2 \pi} A \cdot r\bigl(y(x)\bigr)
                + \sqrt{2 \pi} A \cdot \Bigl(\kappa x
            + \bigl(\zeta + \frac{1}{2}\bigr) x^2\Bigr)\\
                &\quad + \bigl(\sqrt{2 \pi} A - 1\bigr)\sum_{i = 1}^I c_i K_i(x)
                + \sqrt{2 \pi} A \sum_{i = 1}^I (a_i - c_i) K_i(x)
                + (\sqrt{2 \pi}A - 1).
\end{align*}
Note that the absolute value of $g(x)$ can be bounded by the following terms,
\begin{align*}
        \lvert g(x) \rvert
                &\leq
                \sqrt{2 \pi} A\, \lvert r\bigl(y(x)\bigr) \rvert +
                \sqrt{2 \pi} A\, \lvert \kappa x \rvert +
                \sqrt{2 \pi} A\, \lvert \zeta + \frac{1}{2} \rvert x^2\\
                \hskip -1cm & \hskip 1cm +
                \sqrt{2 \pi} A\, \bigl \lvert \sum_{i=1}^I (a_i - c_i) K_i(x) \bigr \rvert +
                \lvert \sqrt{2 \pi} A - 1 \rvert\, \bigl \lvert \sum_{i=1}^I c_i K_i(x) \bigr \rvert+
                \lvert \sqrt{2 \pi}A - 1 \rvert.
\end{align*}
We have,
\begin{align*}
\begin{split}
\gdef\sumciK{\sum_{i=1}^I c_i K_i(x)}
        \abs{f_0(x) - \hat f_0(x)} &=
        \abs{\varphi(x) \cdot g(x)}\\
        &= \varphi(x) \Bigl \lvert \sqrt{2 \pi} A r\bigl(y(x)\bigr)
                + \sqrt{2 \pi} A \Bigl(\kappa x + \bigl(\zeta + \frac{1}{2}\bigr) x^2\Bigr)
                        + (\sqrt{2 \pi} A - 1) \sumciK \Bigr.\\
                        &\hskip 1.5cm \Bigl. + \sqrt{2 \pi} A \sum_{i=1}^I (a_i - c_i) K_i(x)
                        + (\sqrt{2 \pi} A - 1)\Bigr \rvert.
\end{split}
\end{align*}
We now multiply both sides by $e^{\delta x^2}$
and setting $\tilde \delta = \frac{1}{2} - \delta$, so that
$e^{\delta x^2} \varphi(x) = (2 \pi)^{-1/2} e^{-\tilde \delta x^2}$, we have
\begin{align}
\gdef\sumciK{\sum_{i=1}^I c_i K_i(x)}
        \bigl \lvert e^{\delta x^2} \bigl(f_0(x) - \hat f_0(x)\bigr) \bigr \rvert
                &= (2 \pi)^{-1/2} e^{-\tilde \delta x^2}
        \Bigl \lvert \sqrt{2 \pi} A r\bigl(y(x)\bigr)
                + \sqrt{2 \pi} A \Bigl(\kappa x +
        \bigl(\zeta + \frac{1}{2}\bigr) x^2\Bigr)\nonumber\\
                        &\hskip 2.8cm  \Bigl. + (\sqrt{2 \pi} A - 1) \sumciK
                        + \sqrt{2 \pi} A \sum_{i=1}^I (a_i - c_i) K_i(x)\nonumber\\
                        &\hskip 2.8cm + (\sqrt{2 \pi} A - 1)\Bigr
            \rvert\nonumber\\
        &\leq T_1(x) + T_2(x) + \frac{1}{\sqrt{2 \pi}} \cdot T_3(x)
            + T_4(x) + \frac{1}{\sqrt{2 \pi}} \cdot T_5(x),\label{eqf0Hatf0LeqTterms}
\end{align}
where,
\begin{align*}
        T_1(x) &:= \bigl \lvert A e^{- \tilde \delta x^2}
                        r\bigl(y(x)\bigr) \bigr \rvert;\\
        T_2(x) &:= \bigl \lvert A e^{- \tilde \delta x^2}
                        \bigl(\kappa x + (\zeta + \frac{1}{2}) x^2 \bigr) \bigr
                        \rvert;\\
        T_3(x) &:= \bigl \lvert (\sqrt{2 \pi} A - 1) e^{- \tilde \delta x^2}
                        \sum_{i=1}^I c_i K_i(x) \bigr \rvert;\\
        T_4(x) &:= \bigl \lvert A e^{- \tilde \delta x^2}
                        \sum_{i=1}^I (a_i - c_i) K_i(x) \bigr \rvert;\\
    T_5(x) &:= \lvert e^{-\tilde \delta x^2} (\sqrt{2 \pi} A - 1) \rvert.
\end{align*}

If we show that $\norm{T_i}_{\infty}$ is at least of order $\norm{c}^2$ as
$c \rightarrow 0$ for $i = 1, \ldots, 5$, then we can conclude the proof by
taking the supremum of~\eqref{eqf0Hatf0LeqTterms} over $x \in \R$, which gives,
\[
        \sup_{x \in \R} \bigl \lvert e^{\delta x^2} (f_0(x) - \hat f_0(x)) \bigr \rvert
            = \mathcal{O}(\norm{c}^2),
    \]
as $c \rightarrow 0$.

\textbf{Term $\bm{T_1}$}. First, note that
\begin{equation*}
         \bigl \lvert e^{- \tilde \delta x^2} r(y(x)) \bigr \rvert
         \leq \max_{\sigma \in \{-1, 1\}} \bigl \lvert
        e^{- \tilde \delta x^2} r(\sigma \cdot \lvert y(x) \rvert)
        \bigr \rvert, \text{ for all } x \in \R,
\end{equation*}
and thus,
\begin{equation}\label{eqT1Supremum}
        \sup_{x \in \R} \lvert T_1(x) \rvert
            \leq A \cdot
                \sup_{{\substack{x \in \R \\ \sigma \in \{-1, 1\}}}}
            \lvert e^{- \tilde \delta x^2} r(\sigma \cdot \lvert y(x) \rvert)
            \rvert.
\end{equation}
Next we choose $\gamma$ such that,
\begin{equation}\label{eqT1Gamma}
        q_1 :=
        \sup_{{\substack{x \in \R \\ \sigma \in \{-1, 1\}}}}
        \bigl \lvert e^{- \tilde \delta x^2}
        r\bigl(\sigma \cdot \gamma (1 + x^2)\bigr) \bigr \rvert < \infty.
\end{equation}
This is always possible for some $\gamma \in (- \tilde \delta, \tilde \delta)$,
as $r(0) = 0$, and
since $e^{- \tilde
\delta x^2} r\bigl(\pm \gamma (1 + x^2)\bigr)$ is continuous and $r\bigl(\pm \gamma (1 +
x^2)\bigr)$ grows no faster that $e^{\gamma x^2}$ as $x \rightarrow \pm\infty$,
it is beaten by $e^{- \tilde \delta x^2}$ in the tails.\\
For $y(x)$ as given in~\eqref{eqDefnFunctionY} and using~
\eqref{eqKiBoundingConstraint} we can find an upper
bound by
\begin{align}
        \lvert y(x) \rvert
                &\leq \lvert \kappa \rvert \cdot \Bigl(\frac{1 + x^2}{2} \Bigr) + \lvert \zeta
        + \frac{1}{2} \rvert \cdot (1 + x^2) + \sum_{i=1}^I \lvert a_i \rvert B_i (1
        + x^2)\nonumber\\
                &= \gamma (1 + x^2) \cdot
                        \frac{1}{\gamma} \Bigl( \frac{1}{2} \lvert \kappa \rvert
            + \lvert \zeta + \frac{1}{2} \rvert
            + \sum_{i=1}^I \lvert a_i \rvert B_i \Bigr)\nonumber\\
        &=: \gamma (1 + x^2) \cdot \varepsilon_1,\label{eqT1YBoundEpsilon}
\end{align}
where $\gamma$ is such that~\eqref{eqT1Gamma} holds.
As $c \rightarrow 0$, we have by Proposition~\ref{propNormSquaredIFT}, $\kappa
\rightarrow 0$, $\zeta \rightarrow -1/2$ and $a_i \rightarrow c_i$. Therefore, we
can choose $c$ small enough (and depending on $\gamma$)
such that $\varepsilon_1 \in [0,1]$.
Now, from~\eqref{eqT1Supremum},~\eqref{eqT1YBoundEpsilon}, the fact that $r$ is
convex with a minimum at zero, and by
Lemma~\ref{lemmaRxSupremumBound} we get
\begin{align*}
    \sup_{x \in \R} \bigl \lvert T_1(x) \bigr \rvert
                &\leq A \sup_{{\substack{x \in \R \\ \sigma \in \{-1, 1\}}}}
            \bigl \lvert e^{- \tilde \delta  x^2}
            r\bigl( \sigma \gamma (1 + x^2) \varepsilon_1 \bigr) \bigr \rvert\\
                &\leq \varepsilon_1^2 A \sup_{{\substack{x \in \R \\ \sigma \in \{-1, 1\}}}}
            \bigl \lvert e^{- \tilde \delta  x^2}
            r\bigl( \sigma \gamma (1 + x^2) \bigr) \bigr \rvert\\
        &= \varepsilon_1^2\, A \, q_1.
\end{align*}
By Proposition~\ref{propNormSquaredIFT}, we have $A \rightarrow 1 / \sqrt{2 \pi}$ as
$c \rightarrow 0$ and,
\begin{equation*}
        \varepsilon_1 = \frac{1}{\gamma} \Bigl(
            \frac{1}{2} \lvert \kappa \rvert + \lvert \zeta + \frac{1}{2} \rvert
            + \sum_{i=1}^I \lvert a_i \rvert B_i
            \Bigr)
        = \mathcal{O}(\norm{c}), \text{ as } c \rightarrow 0,
\end{equation*}
and therefore $\varepsilon_1^2 = \mathcal{O}(\norm{c}^2)$ as $c \rightarrow 0$, and
$\norm{T_1}_{\infty} = \mathcal{O}(\norm{c}^2)$ as $c \rightarrow 0$.

\textbf{Term $\bm{T_2}$}. We proceed similarly as for $T_1$, and look for some
$\varepsilon_2 \in [0, 1]$ such that
$\lvert \kappa x + (\zeta + \frac{1}{2}) x^2 \rvert \leq \varepsilon_2 (1 + x^2)$.
We have,
\begin{align*}
\bigl \lvert \kappa x + (\zeta + \frac{1}{2}) x^2 \bigr \rvert
    &\leq \lvert \kappa \rvert (\frac{1 + x^2}{2})
        + \lvert \zeta + \frac{1}{2} \rvert (1 + x^2)\\
    &= \bigl(\frac{1}{2} \lvert \kappa \rvert
        + \lvert \zeta + \frac{1}{2} \rvert \bigr) (1 + x^2).
\end{align*}
Setting $\varepsilon_2 := (\frac{1}{2} \lvert \kappa \rvert + \lvert \zeta +
\frac{1}{2} \rvert)$, by Proposition~\ref{propNormSquaredIFT},
$\varepsilon_2 = \mathcal{O}(\norm{c}^2)$ as $c \rightarrow 0$,
and thus we can choose $c$ sufficiently small such that $\varepsilon_2 \leq 1$.
Let,
\begin{equation*}
        q_2 := \sup_{x \in \R} \bigl \lvert
            e^{-\tilde \delta x^2} (1 + x^2) \bigr \rvert < \infty,
\end{equation*}
where $q_2 < \infty$ since $e^{- \tilde \delta x^2} (1 + x^2)$ is
continuous and tends to zero in the tails.
From this, for $\varepsilon_2 \in [0, 1]$ as above, we can apply
Lemma~\ref{lemmaSupremumConvexFunctions} and get
\begin{align*}
    \sup_{x \in \R} \bigl \lvert e^{- \tilde \delta x^2}
        (\kappa x + (\zeta + \frac{1}{2}) x^2) \bigr \rvert
        &\leq A \sup_{x \in \R} \bigl \lvert e^{\tilde \delta x^2}
        \varepsilon_2 (1 + x^2) \bigr \rvert \\
        &\leq \varepsilon_2 \sup_{x \in \R}
        \bigl \lvert e^{-\tilde \delta x^2} (1 + x^2) \bigr \rvert
        = \varepsilon_2 \, q_2.
\end{align*}
Then,
\[
\sup_{x \in \R} \bigl \lvert T_2(x) \bigr \rvert
    = A \sup_{x \in \R} \bigl \lvert e^{\tilde \delta x^2}
        (\kappa x + (\zeta + \frac{1}{2}) x^2) \bigr \rvert
    \leq A \, \varepsilon_2 \, q_2.
\]
Therefore, we have $\norm{T_2}_{\infty} = \mathcal{O}(\norm{c}^2)$, as $c \rightarrow 0$.

\textbf{Term $\bm{T_3}$}. As with the $T_2$ term, we want an
$\varepsilon_3 \in [0, 1]$, such that
$\lvert \sum_{i=1}^I c_i K_i(x) \rvert \leq \varepsilon_3 (1 + x^2)$,
so that we can apply Lemma~\ref{lemmaSupremumConvexFunctions} to show
\begin{equation*}
    \sup_{x \in \R} \bigl \lvert e^{- \tilde \delta x^2}
        \sum_{i=1}^I c_i K_i(x) \bigr \rvert
        \leq \varepsilon_3 \sup_{x \in \R}
            \bigl \lvert e^{-\tilde \delta x^2} (1 + x^2) \bigr \rvert < \infty.
\end{equation*}

First, note that by~\eqref{eqKiBoundingConstraint},
\begin{align*}
\bigl \lvert \sum_{i=1}^I c_i K_i(x) \bigr \rvert
    &\leq \bigl \lvert \sum_{i=1}^I c_i B_i (1 + x^2) \bigr \rvert,\\
    &= \bigl \lvert \sum_{i=1}^I c_i B_i \bigr \rvert \cdot (1 + x^2),
\end{align*}
and thus we set $\varepsilon_3 := \bigl \lvert \sum_{i=1}^I c_i B_i \bigr \rvert$.
Clearly, $\varepsilon_3 = \mathcal{O}(\norm{c})$ as $c \rightarrow 0$.
Now, with $c$ sufficiently small such that $\varepsilon_3 \in [0,1]$, we have
by Lemma~\ref{lemmaSupremumConvexFunctions},
\begin{align*}
\sup_{x \in \R} \bigl \lvert e^{- \tilde \delta x^2}
    \sum_{i=1}^I c_i K_i(x) \bigr \rvert
        &\leq \sup_{x \in \R} \bigl \lvert e^{- \tilde \delta x^2}
            \bigl( \lvert \sum_{i=1}^I c_i B_i \rvert \bigr)
            (1 + x^2) \bigr \rvert\\
    &\leq \lvert \sum_{i=1}^I c_i B_i \rvert
        \cdot \sup_{x \in \R} \bigl \lvert
        e^{- \tilde \delta x^2} (1 + x^2) \bigr \rvert\\
    &\leq \varepsilon_3 \, q_2.
\end{align*}
Therefore,
\begin{align*}
\sup_{x \in \R} \bigl \lvert T_3(x) \bigr \rvert
    & \leq \lvert \sqrt{2 \pi} A - 1 \rvert
        \varepsilon_3 \, q_2
\end{align*}
Thus, $\norm{T_3}_{\infty} = \mathcal{O}(\norm{c}^3)$, as $c \rightarrow 0$,
since $\lvert \sqrt{2 \pi} A - 1 \rvert = \mathcal{O}(\norm{c}^2)$ and
$\varepsilon_3 = \mathcal{O}(\norm{c})$ as $c \rightarrow 0$.

\textbf{Term $\bm{T_4}$}. Similar to the $T_2$ and $T_3$ terms, we want an
$\varepsilon_4 \in [0,1]$ such that
$\sum_{i=1}^I (a_i - c_i) K_i(x) \leq \varepsilon_4 (1 + x^2)$.
Note that
\begin{align*}
\bigl \lvert \sum_{i=1}^I (a_i - c_i) K_i(x) \bigr \rvert
    \leq \bigl \lvert \sum_{i=1}^I (a_i - c_i) B_i \bigr \rvert \cdot (1 + x^2),
\end{align*}
by~\eqref{eqKiBoundingConstraint} and thus we set
$\varepsilon_4 := \lvert \sum_{i=1}^I (a_i - c_i) B_i \rvert$,
and by Proposition~\ref{propNormSquaredIFT},
$\varepsilon_4 = \mathcal{O}(\norm{c}^2)$ as $c \rightarrow 0$.
Choose $c$ small enough such that $\varepsilon_4 \in [0, 1]$. Then, by
Lemma~\ref{lemmaSupremumConvexFunctions},
\begin{align*}
        \sup_{x \in \R} \lvert T_4(x) \rvert
            &\leq A \,\sup_{x \in \R} \bigl \lvert
                e^{- \tilde \delta x^2} \sum_{i = 1}^I (a_i - c_i) B_i (1 + x^2)
                \bigr \rvert\\
            &= \varepsilon_4 \,A \,q_2,
\end{align*}
and since $\varepsilon_4 = \mathcal{O}(\norm{c}^2)$ as $c \rightarrow 0$, we have
$\norm{T_4}_{\infty} = \mathcal{O}(\norm{c}^2)$, as $c \rightarrow 0$.

\textbf{Term $\bm{T_5}$}. Here we can use $e^{\tilde \delta x^2} \leq 1$ for all
$x \in \R$, and from
Proposition~\ref{propNormSquaredIFT} we have
\[
T_5(x) \leq \lvert \sqrt{2 \pi} A - 1 \rvert = \mathcal{O}(\norm{c}^2),
        \text{ as } c\rightarrow 0.
\]
This completes the proof.
\end{proof}

We have therefore shown that for sufficiently small $c$, the
approximation $\hat f_0$ for the density that maximises entropy given
the constraints in~\eqref{eqConstraints} is `close to' $f_0$. We have
also shown that the speed of convergence is of order
$\lVert c \rVert^2$.

\subsection{Step 3.}

We now turn our attention to Step~3.\ of the approximations, where we
find approximations for the entropy and negentropy of $\hat f_0$.  For
these proofs we require that $\hat f_0(x) \geq 0$ for all $x \in \R$,
and thus $\hat f_0$ is a density.

\begin{lemma}[Approximation of Entropy]\label{lemmaEntropyApproximation}
Suppose Assumptions~\ref{assumptionGK} and~\ref{as1K} hold, and let $\hat f_0$ be
given as in Theorem~\ref{theoremSupremumf0Hatf0}. Suppose also
that $\hat f_0(x) \geq 0$ for all $x \in \R$. Then the entropy of $\hat f_0$ satisfies
\begin{equation*}
    H[\hat f_0]
        = \hat H[\hat f_0] + R(\hat f_0),
\end{equation*}
where,
\begin{equation*}
    \hat H[\hat f_0]
        := \eta(1) - \frac{1}{2} \lVert c \rVert^2,
\end{equation*}
with $\eta(\cdot)$ given in~\eqref{eqGaussEntropy},
$c = (c_1, c_2, \ldots, c_I)^\top$, with the
$c_i$ defined in Proposition~\ref{propf0MaxEntropy} and
the remainder term bounded by
\[
    \lvert R(\hat f_0) \rvert
        \leq C \, \tilde M
                \cdot \lVert c \rVert^3,
\]
for some constant $C \in \R$,
and $\tilde M$ given in Assumption~\ref{as1K}.
\end{lemma}

\begin{proof}
Set $K(x) = (K_1(x), K_2(x), \ldots, K_I(x))^\top$, for $x \in \R$.
Now, with $\hat f_0$ as in Theorem~\ref{theoremSupremumf0Hatf0},
expanding $H[\hat f_0]$ gives,
\begin{align*}
\gdef\sumciK{c^\top K(x)}
H[\hat f_0]
    &= -\int \hat f_0(x) \log \hat f_0(x)  \di x\\
    &= -\int \varphi(x) \Bigl(1 + c^\top K(x) \Bigr)
        \Bigl(\log \varphi(x) + \log \bigl(1 + \sumciK \bigr)\Bigr)\\
    &= -\int \varphi(x) \log\varphi(x) \di x
        - \int \varphi(x) \sumciK \log\varphi(x) \di x \\
     &\hskip2cm
        - \int \varphi(x) \bigl(1 + \sumciK \bigr)
            \log \bigl(1 + \sumciK \bigr) \di x \\
    &= \eta(1)
        - \int \varphi(x) \sumciK \bigl(-\frac{1}{2} \log(2 \pi)
        - \frac{1}{2} x^2 \bigr) \di x \\
    &\hskip2cm
        - \int \varphi(x) \bigl(1 + \sumciK \bigr)
            \log \bigl(1 + \sumciK \bigr) \di x\\
    &= \eta(1)
        - 0
        - \int \varphi(x) \bigl(1 + \sumciK \bigr)
            \log \bigl(1 + \sumciK \bigr) \di x,
\end{align*}
using the constraints given in~\eqref{eqKxConditions}.
To obtain the approximation $\hat H[\hat f_0]$ and remainder
$R(\hat f_0)$ terms, we consider the expansion of
$\bigl(1 + c^\top K(x)\bigr) \log\bigl(1 + c^\top K(x)\bigr)$ around
$c = 0$ using the Taylor series.
Let $q(y) = y \log(y)$, $y \in \R$. Then, we have
\[
        q'(y) = \log(y) + 1;\quad
        q''(y) = \frac{1}{y};\quad
        \text{and } q'''(y) = - \frac{1}{y^2}.
\]
and thus using Taylor series around $y_0$ gives
$q(y_0 + h)= h + \frac{1}{2} h^2 + R_1(y_0, h)$,
where $R_1(y_0, h)$ is the remainder term given by
\begin{align*}
    R_1(y_0, h)
        &= \int_{y_0}^{y_0 + h} \frac{(y_0 + h - \tau)^2}{2}
        \Bigl(\frac{-1}{\tau^2}\Bigr) \di \tau\\
        &= - h^3 \int_0^1 \frac{(1 - t)^2}{2 (1 + t h)^2} \di t
\end{align*}
with the change of variables $\tau = (y_0 + t h)$.

Now let us pick $y_0 = 1$ and $h = c^\top K(x)$ and denote by $R_2(x)$
the corresponding remainder $R_2(x) = R_1(1, \,c^\top K(x))$. Then,
\begin{equation}\label{eqHTaylor}
\gdef\sumciK{c^\top K(x)}
H[\hat f_0]
        = \eta(1)
        - \int \varphi(x) \Bigl(\sumciK
        + \frac{1}{2} \bigl(\sumciK \bigr)^2
        + R_2(x) \Bigr) \di x,
\end{equation}
where the remainder term $R_2(x)$ is given explicitly by
\begin{equation*}
    R_2(x)
        = - \bigl( c^\top K(x) \bigr)^3
            \int_0^1 \frac{(1 - t)^2}{2 \bigl( 1 + t c^\top K(x) \bigr)^2}
            \di t.
\end{equation*}
Now using~\eqref{eqKxConditions} and setting
\begin{equation}\label{eqTaylorRemainderInt}
        R(\hat f_0) := - \int_{\R} \varphi(x) \, R_2(x) \di x
\end{equation}
we get from~\eqref{eqHTaylor},
\begin{align*}
        H[\hat f_0]
            &= \eta(1) + 0 - \frac{1}{2} \sum_{i=1}^I c_i^2 + R(\hat f_0)\\
            &= \hat H[\hat f_0] + R(\hat f_0),
\end{align*}
as needed to be shown. It remains to prove the bound for $R(\hat f_0)$.

From Assumption~\ref{as1K}
there exists some $\varepsilon > 0$ such that
$c^\top K(x) \geq -1/2$ for all $c$ with
$c^\top c \leq \varepsilon$ for all $x \in \R$, and therefore,

\begin{align*}
    \lvert R_2(x) \rvert
        &= \Bigl \lvert \bigl( c^\top K(x) \bigr)^3
            \int_0^1 \frac{(t - 1)^2}{2 \cdot (1 + t c^\top K(x))^2}
            \di t
            \Bigr \rvert\\
        &\leq \bigl \lvert \bigl( c^\top K(x) \bigr) \bigr \rvert^3
            \cdot \Bigl \lvert
            \int_0^1 \frac{(t - 1)^2}{2 \cdot (1 - t / 2)^2}
            \di t
            \Bigr \rvert\\
        &= C \cdot \bigl \lvert \bigl( c^\top K(x) \bigr) \bigr \rvert^3,
\end{align*}
where $C \in \R$, as the integral is of a continuous
function over a compact set.

Now, there exists some $\delta > 0$ such that for all $c^\top c \leq \delta$,
$c_i \leq \lVert c \rVert$ for all $i = 1, 2, \ldots, I$. Then, with
$c^\top c \leq \min(\varepsilon, \delta)$, we have
\begin{align*}
    \lvert R_2(x) \rvert
        &\leq C
            \sum_{i,j,k = 1}^I \bigl \lvert K_i(x) K_j(x) K_k(x) \bigr \rvert
            \cdot \lVert c \rVert^3\\
        &\leq C \cdot M(x) \cdot \lVert c \rVert^3
\end{align*}
having used~\eqref{eqKxConditionM-1} from Assumption~\ref{as1K}.
Putting this all together we obtain the bound for $R(\hat f_0)$,
\[
        \bigl \lvert R(\hat f_0) \bigr \rvert
        \leq \int_{\R} \varphi(x) \lvert R_2(x) \rvert \di x\\
        \leq C \tilde M \lVert c \rVert^3,
    \]
where $\tilde M$ is given in~\eqref{eqKxConditionM-2}, as required.
\end{proof}

\begin{remark}
Note that the density $\hat f_0$ has unit variance. Indeed, by~
\eqref{eqKxConditions},
\begin{align*}
        \int \hat f_0(x) x^2 \di x
                &= \int \varphi(x) \bigl(1 + \sum_{i=1}^I c_i K_i(x)\bigr) \di x\\
                &= \int \varphi(x) x^2 \di x
            + \sum_{i=1}^I c_i \int \varphi(x) K_i(x) x^2 \di x
        = 1.
\end{align*}
Therefore, the negentropy equivalent of the entropy approximation given in
Lemma~\ref{lemmaEntropyApproximation} is
$J[\hat f_0] = \hat J[\hat f_0] + R(\hat f_0)$
with $R(\hat f_0)$ given as in~\eqref{eqTaylorRemainderInt} and
\begin{equation}\label{eqNegentropyApproximation}
    \hat J[\hat f_0] = \frac{1}{2} \lVert c \rVert^2
        = \frac{1}{2} \sum_{i=1}^I c_i^2.
\end{equation}
\end{remark}

\begin{proposition}\label{propNegentropyProportional}
With the same assumptions as in Lemma~\ref{lemmaEntropyApproximation}.
Set $I = 1$. Then,
\begin{equation*}
    \hat J[\hat f_0] \propto \bigl(\Ex_f G(Y) - \Ex_\varphi G(Z) \bigr)^2,
\end{equation*}
where $Y$ is a random variable with density $f$ and $Z \sim \mathcal{N}(0,1)$.
\end{proposition}

\begin{proof}
By the constraints that need to be satisfied by $K$, given in Assumption~\ref{as1K}, we have
$\int \varphi(x) K(x) x^k \di x = 0$ for $k = 0, 1, 2$.
Substituting~\eqref{eqKx} for $K(x)$ in~\eqref{eqKxConditions} and solving
these three equations gives an
explicit expression for $\alpha, \beta, \gamma$ in terms of $G$, given by,
\begin{align}\label{eqABGEquations}
    \alpha &= \frac{1}{2}\Bigl(\int \varphi(x) G(x) \di x
            - \int \varphi(x) G(x) x^2 \di x\Bigr);\nonumber\\
    \beta &= -\int \varphi(x) G(x) x \di x;\\
    \gamma &= \frac{1}{2}\Bigl(\int \varphi(x) G(x) x^2 \di x
            - 3 \int \varphi(x) G(x) \di x\Bigr)\nonumber.
\end{align}
Recall that
$c = \Ex K(Y) = \dfrac{1}{\delta} \bigl( \Ex G(Y) + \alpha \Ex Y^2
+ \beta \Ex Y + \gamma \bigr)$.
Now using~\eqref{eqABGEquations} and the fact that $\Ex Y = 0$ and $\Ex Y^2 =
1$ (since $Y$ has density $f$), we get
$c = \dfrac{1}{\delta} \bigl( \Ex G(Y) - \Ex G(Z) \bigr)$.
From~\eqref{eqNegentropyApproximation} with $I = 1$, we have
$\hat J[\hat f_0] = \dfrac{1}{2} c^2$, hence,
\begin{equation*}
    \hat J[\hat f_0]
        = \frac{\bigl(\Ex G(Y) - \Ex G(Z) \bigr)^2}{2\,\delta^2}.
\end{equation*}
This completes the proof, with $C = 2\,\delta^2$ in Step~3.\ of Section~
\ref{secFastIca}. Note that $\delta$ can be found by solving the
additional constraint $\int\varphi(x)K(x)^2\di x = 1$.
\end{proof}

This concludes our discussion of the approximations used in fastICA.  We have shown
that under certain conditions, the approximations given in Steps~2., 3.\ and~4.\ in
Section~\ref{secFastIca} are ``close'' to the true values. We will now give an
example where these approximations are indeed close to one-another, but the
surrogate density of the projections, $f_0$ from Step~1.\, is not close to the true
density~$f$.

\section{Example}
\label{secExample}

We now highlight the approximation steps as explained in Section~\ref{secFastIca} on
a toy example.
In this section we use example data as illustrated in Figure~\ref{figPoints}, which
was intentionally created in a
very simplistic manner to further emphasise the ease at which false
optima are found using the contrast function $\hat J^\ast(y)$. The data
was obtained by pre-selecting vertical columns where no data points
are allowed. An iterative scheme was then employed, as explained
below:
\begin{enumerate}
    \item Sample $n$ points from a standard two-dimensional Gaussian distribution;
    \item Remove all points that lie in the pre-specified columns;
    \item Whiten the remaining $\tilde n$ points;
    \item Sample $n - \tilde n$ points from a standard two-dimensional Gaussian
distribution.
\end{enumerate}
Repeat 2.\ - 4.\ until we have a sample of size $n$ with no points
lying in the pre-specified columns. No optimisation was done to the
distribution of these points to attempt to force the fastICA contrast
function to have a false optimum.

We will use the $m$-spacing approximation~\eqref{m-spacing} to obtain a contrast function
that can be compared to the fastICA contrast function~\eqref{eqHatJStarY}.
Following \citet{learned2003ica},
we chose $m = \sqrt{n}$, where $n \in \mathbb{N}$ is the number of
observations. This was chosen so that the condition $m / n \rightarrow 0$ as $n
\rightarrow \infty$ is satisfied \citep{vasicek1976test, beirlant1997nonparametric}.
This approximation to entropy is a direct approximation to $H[f]$, and
therefore does not involve an equivalent Step~1.\ from Section~\ref{secFastIca} where
$f$ is substituted by a new density $f_0$.

Using the $m$-spacing method to find the first independent component loading,
we want to obtain the direction
$w^\ast := \argmin_{w \in \R^p, w^\top w = 1} H_{m, n}(Dw)$. In
the example of this paper, numerical minimisation is used to obtain $w^\ast$ and the
associated projection $D w^\ast$. The contrast function to compare against the
fastICA contrast
function~\eqref{eqHatJStarY}
is given by the $m$-spacing negentropy approximation, $J_{m,n}(y) = \eta(1) -
H_{m,n}(y)$ for directions $w \in \R^n$ on the half-sphere. Note that
$w^\ast = \argmax_{w \in \R^p, w^\top w = 1} J_{m, n}(Dw)$.
In general this contrast function is not very smooth, although a method to attempt
to overcome this non-smoothness (and the resulting local optima, which can cause
numerical optimisation issues) is given in \citet{learned2003ica},
and involves replicating the data with some added Gaussian noise.

To illustrate the kind of problems which can occur during the
approximation from $f$ to $\hat f_0$ and from $J[f]$ to
$\hat J^\ast(y)$, we construct an example where the density $f$ in the
direction of maximum negentropy is significantly different to
$\hat f_0$ in the same direction. This results in \texttt{fastICA}
selecting a sub-optimal projection, as shown below.
Here we just consider the case
$I = 1$ in Assumption~\ref{assumptionGK}, with one $G = G_1$ and thus one $K=K_1$.
Moreover, in \texttt{fastICA} there is a choice of two functions to use,
$G(x) := (1 / \alpha) \log \cosh(\alpha x)$, $\alpha \in [1, 2]$, and
$G(x) := -\exp(- x^2 / 2)$.
We have considered these two functions with varying alpha, as well as the
fourth moment contrast function given in \citet{miettinen2015fourth}.
Here, the function for Step~3.\ in Section~\ref{secFastIca} is
$\lvert \Ex_f X^4 - 3 \rvert$, and the empirical approximation of the expectation is
used for Step~4., such that the approximate contrast function is
$\lvert \frac{1}{n} \sum_{i = 1}^n y_i^4 - 3 \rvert$.
In the example of this paper all choices give very similar results and
thus we only show the fastICA contrast function resulting from
$G(x) = (1 / \alpha) \log \cosh(\alpha x)$, with $\alpha = 1$ for simplicity.

With the data distributed as in Figure~\ref{figPoints}, the negentropy
over projections in the directions $w_{\theta} = (\sin(\theta), \cos(\theta))$ with
$\theta \in [0, \pi)$ found by the $m$-spacing approximation
and used in the \texttt{fastICA} method is shown in
Figure~\ref{figObjectives}. The contrast function obtained by approximating
$J[f_0]$ directly is also included as the dashed line. The three contrast
functions have been placed below one-another in the order of approximations given in
Figure~\ref{figFastIcaApproximations} and so the $y$-axis is independent for each.
The search is only performed
on the half unit circle, as projections in directions
$w_1 = (\sin(\theta), \cos(\theta))$ and
$w_2 = (\sin(\theta + \pi), \cos(\theta + \pi))$
for any $\theta \in [0, \pi)$
have a reflected density with the same entropy. It is clear from
Figure~\ref{figObjectives} that the \texttt{fastICA} result $\hat J^\ast$ is poor,
with the \texttt{fastICA} contrast function missing the peak of
negentropy that appears when using $m$-spacing. The contrast function
used in the \texttt{fastICA} method clearly differentiates between the
direction of the maximum and other directions, and thus in this
example it is both confident and wrong (since there is a clear and
unique peak). This is also true of the direct approximation to $J[f_0]$, showing
that issues occur at the first step of approximations, when $J[f_0]$ is used instead
of $J[f]$.

\begin{figure}
  \centering
  \input{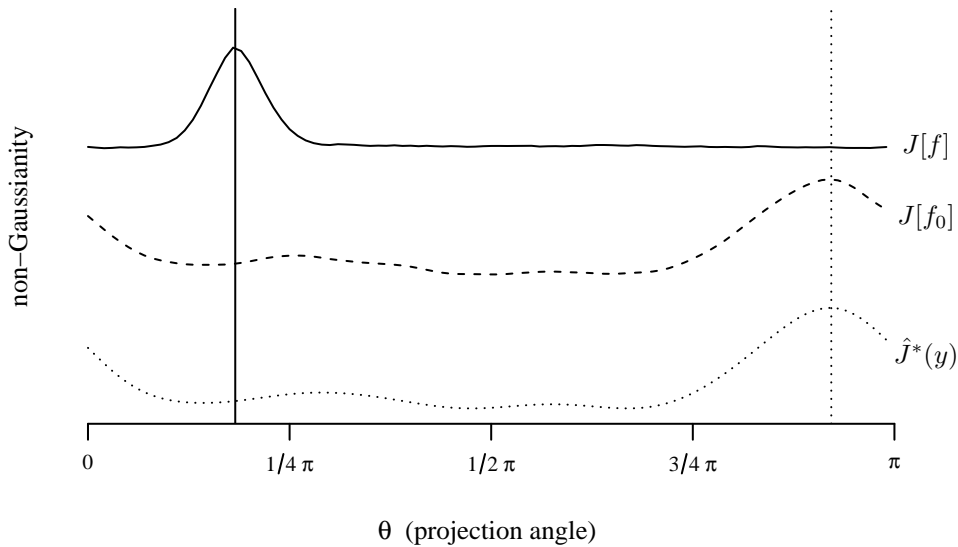}
  \caption{\label{figObjectives}Objective functions of $m$-spacing (solid line),
        $J[f_0]$ (dashed line) and \texttt{fastICA}
        method (dotted line) for projections of the data given in
        Figure~\ref{figPoints} in the directions $\theta \in [0, \pi )$.
        These correspond to $J[f]$, $J[f_0]$ and $\hat J^\ast(y)$ in Figure~\ref{figFastIcaApproximations}.
        The vertical lines give the directions which maximise the contrast functions
        for $m$-spacing (solid line) and \texttt{fastICA} (dotted line).}
\end{figure}

As is shown in Section~\ref{secProofs}, for sufficiently small $c$,
the approximation for the density $\hat f_0$ (given in~
\eqref{eqhatf0}) is ``close to'' $f_0$ (given in~\eqref{eqf0}), and the
speed of convergence is of order $c^2$ for $c \rightarrow
0$. Therefore, it is our belief (backed up by computational
experiments) that the majority of the loss of accuracy occurs in the
approximation step where the surrogate $f_0$ is used instead of $f$,
rather than in the later estimation steps for $J[\hat f_0]$ and
$\hat J^\ast(y)$. This can be seen by comparing numerically the contrast
functions $J[f]$, $J[f_0]$ and $\hat J^\ast(y)$ (shown in
Figure~\ref{figObjectives}), and by comparing the
densities $f$, $f_0$ and $\hat f_0$.
Here, $J[f_0]$ and $\hat J^\ast(y)$ give similar directions
for the maximum, and these differ significantly from the location of
the maximum of $J[f]$. This is a fundamental theoretical problem with
the fastICA method, and is not a result of computational or
implementation issues with \texttt{fastICA}.  In particular, the fact that
the dotted vertical line in Figure~\ref{figObjectives} is at the maximum
of $\hat J^\ast(y)$ indicates that the effect is not a convergence
problem in the \texttt{fastICA} implementation.

\section{Conclusions}
\label{secConclusion}

In this paper we have given an example where the fastICA method misses
structure in the data that is obvious to the naked eye. Since this
example is very simple, the fastICA result is concerning, and this
concern is magnified when working in high dimensions as visual
inspection is no longer easy. There is clearly some issue with the
contrast function (surrogate negentropy) used in fastICA. Indeed,
this surrogate has the property of being an approximation of a {\em
lower bound} for negentropy, and this does not necessarily capture
the actual behaviour of negentropy over varying projections since we
want to maximise {\em negentropy}. To strengthen the claim that
accuracy is lost when substituting the density with the surrogate, we
have shown convergence results for all the approximation steps used in
the method.

To conclude this paper, we ask the following questions which could
make for interesting future work: {\em Is there a way, a priori, to
 know whether fastICA will work?} This is especially pertinent when
fastICA is used with high dimensional data. The trade-off in accuracy
for the fastICA method comes at the point where the density $f$ is
substituted with $f_0$. Therefore one could also ask: {\em Are there
other methods similar to that of fastICA but that use a different
surrogate density which more closely reflects the true projection
density?}

If these two options are not possible, then potentially a completely
different method for ``fast'' ICA is needed, one that either gives a
``good'' approximation for all distributions, or where it is known
when it breaks down. An initial step in this direction can be found
in \citet{smith2020cluster}. In this work the authors propose a new
ICA method, known as clusterICA, using the $m$-spacing approximation
for entropy discussed in this paper, combined with a clustering
procedure.

\section*{Acknowledgement}

We thank the Editor, Associate Editor and referees for their helpful comments.
P.~Smith was funded by NERC DTP SPHERES, grant number NE/L002574/1.

\bibliographystyle{plainnat}
\bibliography{refs}

% \section*{References}

%\bibliographystyle{myjmva.bst}
%\bibliography{refs}

\end{document}